\documentclass[letterpaper, 10 pt, conference]{ieeeconf}

\IEEEoverridecommandlockouts
\usepackage{graphics}
\usepackage{epsfig}
\usepackage{times}
\usepackage{amsmath}
\usepackage{amssymb}
\usepackage{xcolor}
\usepackage{xspace}
\usepackage{booktabs}

\newcommand{\algoname}{\texttt{FAITH}}
\newcommand{\algo}{\algoname}

\definecolor{MaterialBlue900}{HTML}{0D47A1}
\definecolor{MaterialPurple900}{HTML}{4A148C}
\newcommand{\ourname}{\textcolor{MaterialBlue900}{\algo{}}}
\newcommand{\baselinename}[1]{\textcolor{MaterialPurple900}{\texttt{#1}}}

\definecolor{plgrey}{HTML}{888888}

\definecolor{gtsetgrey}{HTML}{797979}
\newcommand{\gtset}[1]{{\color{gtsetgrey}#1}}

\definecolor{obstred}{HTML}{D03B3B}

\newcommand{\R}{\mathbb{R}}

\newcommand{\stt}{s}
\newcommand{\act}{a}

\newcommand{\rew}{r}

\newcommand{\E}{\mathbb{E}}

\newcommand{\harm}{h}

\newcommand{\cbfAlpha}{\alpha}

\DeclareMathOperator*{\argmin}{arg\,min}

\DeclareMathOperator{\clip}{clip}
\newcommand{\abs}[1]{\left\lvert #1 \right\rvert}

\providecommand{\coloneqq}{\mathrel{\vcenter{\hbox{$:$}}}\!=}
\newcommand{\relu}[1]{\left[ #1 \right]_{+}}

\newtheorem{lemma}{Lemma}
\newtheorem{remark}{Remark}

\usepackage{cleveref}

\crefname{equation}{}{}
\Crefname{equation}{Equation}{Equations}
\crefname{problem}{Problem}{Problems}
\crefname{definition}{Definition}{Definitions}
\crefname{assumption}{Assumption}{Assumptions}
\crefname{proposition}{Proposition}{Propositions}
\crefname{theorem}{Theorem}{Theorems}
\crefname{lemma}{Lemma}{Lemmas}
\crefname{remark}{Remark}{Remarks}
\crefname{section}{Section}{Sections}
\crefname{subsection}{Section}{Sections}
\crefname{figure}{Fig.}{Figs.}
\Crefname{figure}{Fig.}{Figs.}
\crefname{table}{Table}{Tables}
\Crefname{table}{Table}{Tables}

\makeatletter
\let\@IEEEorigmakecaption\@makecaption
\long\def\@makecaption#1#2{\ifx\@captype\@IEEEtablestring
    \setbox\@tempboxa\hbox{\footnotesize #1.~~ #2}\ifdim \wd\@tempboxa >\hsize
      \setbox\@tempboxa\hbox{\footnotesize #1.~~ }\parbox[t]{\hsize}{\footnotesize\noindent\unhbox\@tempboxa #2}\else
      \hbox to\hsize{\footnotesize\hfil\box\@tempboxa\hfil}\fi
    \@IEEEtablecaptionsepspace
  \else
    \@IEEEorigmakecaption{#1}{#2}\fi}
\makeatother

\title{\LARGE \bf
FAITH: Feasibility-Aware Safety-Filtered RL for \\High-Dimensional Systems
}

\author{Songyuan Zhang$^{1,2,*}$, Baljeet Singh$^{2}$, Sarthak Ranjeet Kaingade$^{2}$, Chuchu Fan$^{1}$, and Bryan Trinh$^{2}$\thanks{$^{*}$ Work done while interning at Amazon.}\\[0.3em]
{$^{1}$Massachusetts Institute of Technology \quad $^{2}$Amazon}}

\makeatletter
\newcommand{\ADDTEASERFIGURE}{\begin{minipage}{\textwidth}
    \centering
    \includegraphics[width=\textwidth]{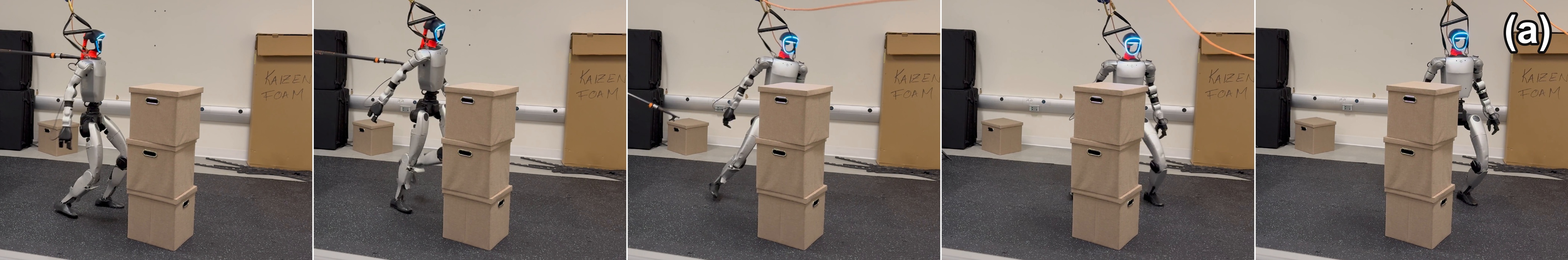}\\[2pt]
    \includegraphics[width=\textwidth]{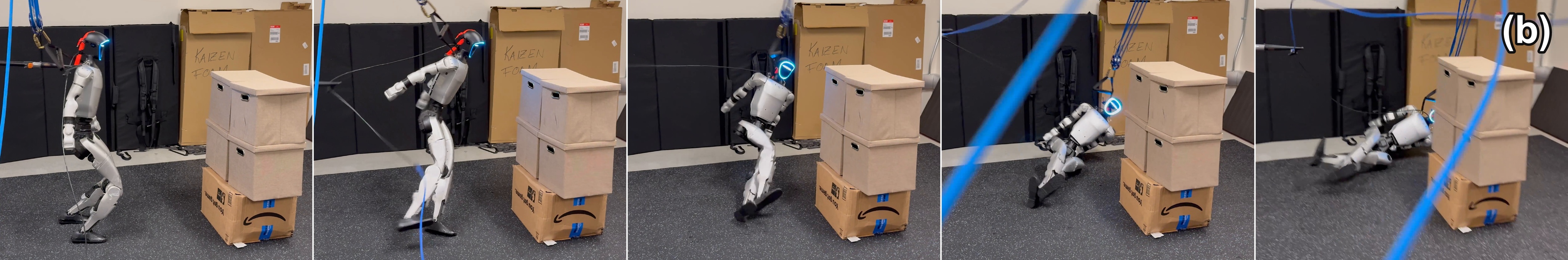}
    \setcounter{figure}{0}\refstepcounter{figure}\def\@captype{figure}\@makecaption{\fnum@figure}{\textbf{\algo{} on a $29$-DoF humanoid under a push, keeping clear of a person represented by the boxes.} {(a)} Under a light push, the robot stabilizes while staying clear of the human. {(b)} Under a strong push, the robot sacrifices balance, falls away from the human, and protects the head and hands.}\label{fig:teaser}
  \end{minipage}}
\AddToHook{cmd/@maketitle/after}{\ADDTEASERFIGURE}
\makeatother

\begin{document}

\maketitle
\thispagestyle{empty}
\pagestyle{empty}

\begin{abstract}

Safe reinforcement learning commonly places safety and task performance in the same policy objective, where they can introduce competing updates. Safety filters separate them at action execution, but classical designs require an analytic safety function and dynamics model, and standard minimal-intervention filters are myopic to long-horizon task return because they minimize only instantaneous action deviation. Hard projections are also undefined when no safe action exists. We present \algo{}, a feasibility-aware, model-free framework that approximates the optimal state-action safety value and amortizes minimal-intervention filtering with a feedforward network. The task policy optimizes the task return through the filtered dynamics, which recovers the feasible constrained problem without a competing safety term in the task-policy update. When no action satisfies the learned safety condition, the same filter approaches the action with minimum predicted peak harm. On a double integrator example and a Safety Gym environment, \algo{} achieves the highest return among methods with no feasible-start violations and matches the lowest harm from infeasible starts. On a $29$-DoF humanoid, it reaches a $99.95\%$ safety rate while retaining $97\%$ of the unfiltered return in Walking-Avoid, and obtains the highest measured safety rate in Push-Avoid by learning to sacrifice balancing and fall away from the protected region. The same policies are also demonstrated on a real-world Unitree G1 humanoid.

\end{abstract}

\section{Introduction}
\label{sec:intro}

Reinforcement learning (RL) has enabled agile behaviors for high-dimensional legged and humanoid robots \cite{he2024agile,yang2025shield,chen2025dexterous}.
Deploying these behaviors around people or equipment, however, requires keeping every robot link outside protected regions.
This requirement is difficult for contact-rich, underactuated systems because safety depends not only on the current geometry, but also on whether future actions can prevent a violation eventually.
A strong disturbance can also make violation unavoidable, and therefore, a controller should maximize task return safely when safety is feasible and only minimize harm otherwise.

Feasibility-aware RL uses reachability-based safety values to estimate whether a safe future trajectory exists \cite{fisac2019bridging,hsu2021safety}, and guides policy optimization toward harm reduction when safety is infeasible, and return maximization otherwise \cite{yu2022reachability,ganai2023respo,zheng2024safe}.
Safety requirements are often enforced within policy optimization through a penalty, a multiplier, or a constraint-reduction update \cite{ray2019benchmarking,xu2021crpo}.
However, task and safety signals act on the policy itself in all cases.
When return improvement and harm reduction favor different actions, their policy updates can oppose one another.
For weighted penalty methods, a small safety weight can permit violations, whereas a large weight can make the policy conservative.
Primal-dual methods adapt the weight, but replace manual tuning with an unstable saddle-point problem \cite{xu2021crpo}.
Thus, even when the formulation distinguishes the feasibility of states, safety competes with performance inside policy optimization.

Safety filters offer a natural separation by correcting a nominal action generated by a task policy before it reaches the system \cite{ames2016control,wabersich2023data}.
Classical filters typically require an analytic safety function and a dynamics model, then solve an optimization problem at every control step.
Learning-based methods can remove the explicit model, but they usually retain online action sampling or optimization \cite{oh2025safety,kim2026deep}, which becomes difficult as the action dimension grows.
In addition, a fixed task policy cannot adapt to the effect of safety interventions on long-horizon return.
Some methods address this limitation by training the task policy through the filter, allowing it to adapt to safety interventions \cite{bejarano2024safety,yang2025cbf,oh2026provably}.
A separate difficulty arises when no admissible action exists: hard projection is then undefined and does not by itself specify how to minimize harm.

Considering these challenges, we propose \textbf{F}easibility-\textbf{A}ware safety-f\textbf{I}l\textbf{T}ered RL for \textbf{H}igh-dimensional systems (\algo{}), a model-free framework that combines learned safety value functions, feedforward action filtering, and task-policy training through the filter.
\algo{} approximates the optimal state-action safety value using a safety critic and a backup actor, which guides a feedforward filter to approximate the minimal action correction.
The task policy is trained \textit{through} the filter using the unmodified task reward.
Safety therefore changes the dynamics experienced by the policy but does not appear as a competing term in its \textit{objective}.
When no action satisfies the learned acceptance condition, the same filter prioritizes the action with minimum peak harm, without a feasibility classifier, a mode switch, or a separate recovery controller.
At deployment, control requires only forward passes through the task policy and safety filter.
To summarize our contributions,
\begin{itemize}
  \item We introduce \algo{}, which approximates the optimal safety critic and trains the task policy through a learned feedforward safety filter.
  Under a perfect filter and global policy optimization, \algo{} recovers the constrained optimum on feasible states, without a competing safety term in the task-policy objective.
  \item In infeasible states, the same controller approaches the action with minimum predicted peak harm, without a mode switch or fallback policy.
  \item Experiments on a double integrator, Safety Gym, and a $29$-DoF humanoid demonstrate high task performance and safety in feasible settings, and reduced harm in infeasible settings.
  Qualitative demonstrations also show the learned controller operating on a physical humanoid.
\end{itemize}

\section{Problem Setting and Preliminaries}
\label{sec:prelim}

\subsection{Feasibility-Aware Safe Reinforcement Learning}
\label{sec:prelim:mdp}

We consider a discrete-time system $s^{k+1}=f(s^k,a^k)$ with state $s\in\mathcal S$, action $a\in\mathcal A$, and \textit{unknown} deterministic dynamics $f$. A reward $r:\mathcal S\times\mathcal A\to\mathbb R$ specifies the task. For a policy $\pi:\mathcal S\to\mathcal A$, define its discounted return as\footnote{Trajectories under $\pi$ satisfy $s^{k+1}=f(s^k,\pi(s^k))$. We omit these dynamics when they are clear from context.}
\begin{equation}
\label{eq:reward-value}
V_r^\pi(s)\coloneqq\sum_{k=0}^{\infty}\gamma^k r(s^k,\pi(s^k)),
\qquad s^0=s,
\end{equation}
where $\gamma\in[0,1)$ and the return is assumed finite. The task objective for a parameterized policy $\pi_\theta$ is then
\begin{equation}
\label{eq:task-obj}
J(\theta)\coloneqq\mathbb E_{s^0\sim\rho_0}\big[V_r^{\pi_\theta}(s^0)\big],
\end{equation}
where $\rho_0$ is the initial-state distribution.

Safety is specified by a signed \textit{harm} function $h:\mathcal S\to\mathbb R$. The unsafe set is then defined as $\mathcal S_u=\{s\in\mathcal S:h(s)>0\}$, with positive values measuring violation severity. For trajectories starting at $s$, define the policy's \textit{peak harm} and the \textit{optimal safety value} \cite{fisac2019bridging} as
\begin{equation}
\label{eq:safety-value}
V_h^\pi(s)\coloneqq\sup_{k\geq0}h(s^k),
\qquad V_h^\star(s)\coloneqq\min_\pi V_h^\pi(s).
\end{equation}
We define the \textit{feasible set} accordingly as $\mathcal F=\{s\in\mathcal S:V_h^\star(s)\leq0\}$. Thus, $s\in\mathcal F$ when \textit{some} policy avoids $\mathcal S_u$ indefinitely. Note that because of the dynamics and actuation limits, $h(s)\leq0$ alone does not imply $s\in\mathcal F$.

Following \cite{yu2022reachability,zheng2024safe}, the ideal feasibility-aware problem distinguishes two regimes for each initial state $s^0$:
\begin{subequations}\label{eq:two-regime}
\begin{alignat}{2}
\max_\pi\ &V_r^\pi(s^0)\quad\text{s.t.}\ V_h^\pi(s^0)\leq0,
&\quad&\text{if }s^0\in\mathcal F,\label{eq:regime-feasible}\\
\min_\pi\ &V_h^\pi(s^0),
&\quad&\text{if }s^0\notin\mathcal F.\label{eq:regime-infeasible}
\end{alignat}
\end{subequations}
Inside $\mathcal F$, we maximize return subject to safety. Outside it, violation is unavoidable, so we minimize peak harm.

\subsection{Enforcing Safety in the Objective}
\label{sec:prelim:lagrangian}

To relate existing methods to \cref{eq:two-regime}, we distinguish whether they enforce safety within task-policy optimization or on the executed action.
Objective-based methods incorporate constraint satisfaction into policy optimization. \textit{Penalty} methods replace the reward by $r_\beta(s,a)=r(s,a)-\beta[h(s)]_+$ for $\beta>0$ and $[\cdot]_+\coloneqq\max\{\cdot,0\}$. \textit{Primal-dual} methods \cite{yu2022reachability} introduce Lagrange multipliers $\lambda$ and jointly optimize $\lambda$ with the policy, introducing the saddle-point problem $\min_{\lambda\geq0}\max_\theta\,
\mathbb E_{s\sim\rho_0}\big[V_r^{\pi_\theta}(s)-\lambda(s)V_h^{\pi_\theta}(s)\big]$.
For infeasible states, this formulation recovers \eqref{eq:two-regime} as $\lambda\to\infty$.
\textit{Primal} methods \cite{xu2021crpo} instead alternate between reward improvement and constraint reduction.

In all three approaches, safety and performance compete \textit{within task-policy optimization}.
Penalty and primal-dual methods weight reward improvement against constraint reduction, while primal methods alternate policy updates between them.
When the two objectives favor opposing update directions, improving safety can counteract task progress.

\subsection{Enforcing Safety on the Action}
\label{sec:prelim:filter}

A \textit{safety filter} instead enforces safety on executed actions, leaving the task objective unchanged \cite{wabersich2023data}. Let $B:\mathcal S\times\mathcal A\to\mathbb R$ encode the safety with $B(s,a)\leq0$ defining admissible actions. Given a nominal action $a_{\mathrm{nom}}=\pi_{\mathrm{task}}(s)$ proposed by a task policy, the \textit{minimal-intervention} filter selects
\begin{equation}
\label{eq:min-intervention}
F(s,a_{\mathrm{nom}})\in\argmin_{a\in\mathcal A}\|a-a_{\mathrm{nom}}\|^2
\quad\text{s.t.}\ B(s,a)\leq0.
\end{equation}
A valid $B$ ensures that starting from $s^0\in\mathcal{F}$, repeatedly applying admissible actions keeps the system in $\mathcal{F}$.

A common approach constructs $B$ from a control barrier function (CBF) and a known continuous-time control-affine model, so \cref{eq:min-intervention} becomes a CBF-based quadratic program \cite{ames2016control}. While constructing a valid CBF is nontrivial, subsequent work learns CBFs \cite{so2024train} or constructs them from reduced-order models \cite{cohen2024safety,yang2025cbf}. Learned state-action safety values also enable filtering without an explicit dynamics model \cite{oh2025safety}.

However, this separation leaves three issues. First, minimizing instantaneous action deviation does not optimize task return. A fixed nominal policy can therefore be suboptimal after filtering. Second, a learned nonlinear $B$ can make online optimization or sampling costly in high-dimensional action spaces \cite{oh2025safety}. Third, if no action satisfies the constraint, the projection does not define a least-harm action.

\section{Method}
\label{sec:method}

To address these limitations, we propose \algo{}.
\algo{} approximates the optimal safety critic with a backup actor supplying low-harm actions for the Bellman target. A learned filter uses this critic to approximate the solution to \eqref{eq:min-intervention} when the learned safety condition is feasible and to prioritize lower harm otherwise. The task policy is trained from filtered rollouts using \textit{only} reward, so safety does not appear as a competing objective in policy updates. \algo{} requires neither known dynamics nor a given barrier, and deployment uses only the task policy and filter. We summarize \algo{} in \Cref{fig:overview} and elaborate its components below.

\definecolor{figtaskblue}{HTML}{1668C4}
\definecolor{figfilterpurple}{HTML}{7A2FBF}
\definecolor{figcriticred}{HTML}{B65047}
\definecolor{figbackupgreen}{HTML}{158568}
\begin{figure*}[t]
  \centering
  \includegraphics[width=\textwidth]{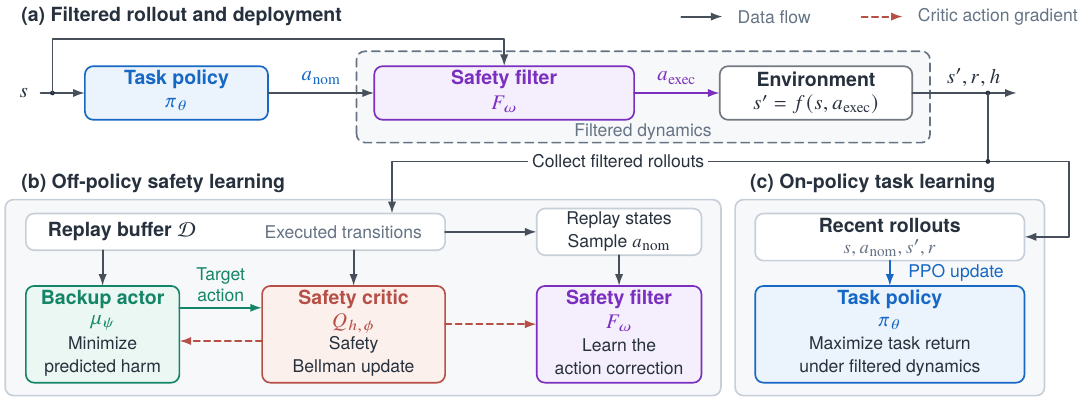}
  \setlength{\abovecaptionskip}{-15pt}
  \caption{\textbf{\algo{} separates safety learning from task-return optimization.} {(a)} The \textcolor{figtaskblue}{task policy} proposes a nominal action, and the \textcolor{figfilterpurple}{safety filter} produces the executed action. The dashed box denotes the filtered dynamics seen by the \textcolor{figtaskblue}{task policy}. {(b)} The \textcolor{figcriticred}{safety critic}, \textcolor{figbackupgreen}{backup actor}, and \textcolor{figfilterpurple}{safety filter} learn from replay. The \textcolor{figbackupgreen}{backup actor} supplies actions for the \textcolor{figcriticred}{safety critic's} Bellman target, while \textcolor{figcriticred}{critic gradients} guide the \textcolor{figbackupgreen}{backup actor} and \textcolor{figfilterpurple}{safety filter}. {(c)} PPO updates the \textcolor{figtaskblue}{task policy} using nominal-action likelihood ratios and filtered-rollout returns. Only the \textcolor{figtaskblue}{task policy} and \textcolor{figfilterpurple}{safety filter} are evaluated for deployment.}
  \label{fig:overview}
\end{figure*}

\subsection{Safety $Q$-Function}
\label{sec:method:critic}

A state value identifies feasible states but cannot evaluate candidate actions without accessing the dynamics. We therefore consider the optimal safety critic \cite{oh2025safety}
\begin{equation}
\label{eq:qh-def}
Q_h^\star(s,a)\coloneqq\min_\pi\sup_{k\geq0}h(s^k),
\qquad s^0=s,\ a^0=a,
\end{equation}
where $a^k=\pi(s^k)$ for $k\geq1$. An action $a$ admits a persistently safe trajectory from $s$ if $Q_h^\star(s,a)\leq0$, and $V_h^\star(s)=\min_{a\in\mathcal A}Q_h^\star(s,a)$ recovers the optimal safety value in \cref{eq:safety-value}. The optimal safety critic satisfies the Bellman equation
\begin{equation}
Q_h^\star(s,a)=\max\!\left\{h(s),\min_{a'\in\mathcal A}Q_h^\star(f(s,a),a')\right\}.
\end{equation}
Let $Q_{h,\phi}$ be its neural approximation. We approximate the $\min_{a'\in\mathcal A}$ with a backup actor $\mu_\psi:\mathcal S\to\mathcal A$, trained by
\begin{equation}
\label{eq:mu-loss}
\mathcal L_\mu(\psi)=\mathbb E_{s\sim\mathcal D}
\left[Q_{h,\phi}\big(s,\mu_\psi(s)\big)\right],
\end{equation}
with samples from the replay buffer $\mathcal D$.
The backup actor supplies the target action and the critic is fitted with
\begin{gather}
y=(1-\gamma_h)h(s)+\gamma_h\max\!\left\{h(s),Q_{h,\bar\phi}\big(s',\mu_\psi(s')\big)\right\},
\label{eq:ra-bellman}\\
\mathcal L_Q(\phi)=\mathbb E_{\mathcal D}
\left[\big(Q_{h,\phi}(s,a)-y\big)^2\right],
\label{eq:critic-loss}
\end{gather}
where $\gamma_h\in[0,1)$ is the safety discount and $Q_{h,\bar\phi}$ is a target critic updated by Polyak averaging. This update targets the discounted safety Bellman fixed point \cite{fisac2019bridging,hsu2021safety}, which can differ from $Q_h^\star$ when $\gamma_h<1$.
Consequently, together with approximation errors, this makes $Q_{h,\phi}(s,a)\leq0$ a learned practical safety estimate instead of a formal certificate. Practically, we can use ensembles \cite{ghasemipour2022so} and target policy smoothing \cite{fujimoto2018addressing} to improve the accuracy of $Q_{h,\phi}$. We sample from a large replay buffer $\mathcal D$ to retain unsafe and near-boundary transitions as the task policy changes,\footnote{In the humanoid experiments, we also penalize deviations of the backup action from replayed executed actions to reduce extrapolation error \cite{levine2020offline}.} and will show its construction in \Cref{sec:method:policy}.

\subsection{Feedforward Safety Filter}
\label{sec:method:filter}

The learned critic defines the action-acceptance condition $Q_{h,\phi}(s,a)+\delta\leq0$, where $\delta\geq0$ is a margin. However, solving the corresponding nonlinear projection online can be expensive in a high-dimensional action space. We instead amortize the projection with a feedforward residual network,
\begin{equation}
\label{eq:filter}
a_{\mathrm{exec}}=F_\omega(s,a_{\mathrm{nom}})
\coloneqq\clip_{\mathcal A}\!\left(a_{\mathrm{nom}}+\Delta_\omega(s,a_{\mathrm{nom}})\right),
\end{equation}
where $a_{\mathrm{nom}}$ is a proposed action and $a_{\mathrm{exec}}$ is applied to the system. Let $\nu$ denote the training distribution formed from replay states and a mixture of replayed and randomly sampled nominal actions. We train the filter with
\begingroup
\thinmuskip=2mu \medmuskip=1mu \thickmuskip=2mu
\renewcommand{\mintagsep}{2pt}
\begin{gather}
\label{eq:filter-loss}
\mathcal L_F(\omega)=\mathbb E_\nu\big[\|a_{\mathrm{exec}}-a_{\mathrm{nom}}\|^2+\lambda[Q_{h,\phi}(s,a_{\mathrm{exec}})+\delta]_+\big],
\end{gather}
\endgroup
where $(s, a_\mathrm{nom})\sim\nu$.
We initialize $F_\omega$ near the identity and increase $\lambda$ during training up to $\lambda_\mathrm{max} \gg 1$, so the filter prioritizes safety as training proceeds.
Let $\mathcal K_\delta(s)\coloneqq\{a\in\mathcal A:Q_{h,\phi}(s,a)+\delta\leq0\}$ be the learned admissible-action set. Eq.~\eqref {eq:filter-loss} is a penalty relaxation of projection onto $\mathcal K_\delta(s)$. If $a_{\mathrm{nom}}\in\mathcal K_\delta(s)$, the ideal filter preserves it. Otherwise, $F$ prioritizes satisfaction of the learned constraint, while the deviation term selects a minimally modified action. The parameterized filter approximates this objective over $\nu$.

Eq.~\eqref {eq:filter-loss} also defines the filter when its learned acceptance condition is infeasible. If $Q_{h,\phi}(s,a)+\delta>0$ for all $a\in\mathcal A$, the hinge is active throughout $\mathcal A$, and Eq.~\eqref {eq:filter-loss} reduces to
$\|a-a_{\mathrm{nom}}\|^2+\lambda Q_{h,\phi}(s,a)$.
Increasing $\lambda$ therefore prioritizes lower predicted peak harm without switching to a separate recovery controller. More precisely, assume that $\mathcal A$ is compact, $Q_{h,\phi}(s,\cdot)$ is continuous, and the per-input loss is minimized globally. If $a_\lambda$ minimizes Eq.~\eqref {eq:filter-loss} and $a_Q\in\argmin_{a\in\mathcal A}Q_{h,\phi}(s,a)$, then
\begin{equation}
0\leq Q_{h,\phi}(s,a_\lambda)-Q_{h,\phi}(s,a_Q)
\leq\operatorname{diam}(\mathcal A)^2/\lambda.
\end{equation}
The bound shows that a global per-input minimizer approaches the minimum of $Q_{h,\phi}(s,\cdot)$ at rate $\mathcal O(1/\lambda)$. For the learned filter, it holds up to learning errors.
In this way, Eq.~\eqref {eq:filter-loss} remains well-defined regardless of feasibility.

\subsection{Task Policy and Training Loop}
\label{sec:method:policy}

The last component of \algo{} is the \textit{task} policy $\pi_\theta$, which proposes $a_{\mathrm{nom}}\sim\pi_\theta(\cdot| s)$, and the filter returns $a_{\mathrm{exec}}=F_\omega(s,a_{\mathrm{nom}})$. For a fixed filter, these maps define the filtered dynamics and reward
\begin{equation}
\label{eq:filtered-dyn}
\begin{aligned}
\tilde f_\omega(s,a_{\mathrm{nom}})
&\coloneqq f\big(s,F_\omega(s,a_{\mathrm{nom}})\big),\\
\tilde r_\omega(s,a_{\mathrm{nom}})
&\coloneqq r\big(s,F_\omega(s,a_{\mathrm{nom}})\big).
\end{aligned}
\end{equation}
The task policy maximizes return in this filtered system,
\begin{equation}
\label{eq:filtered-objective}
J_\omega(\theta)\coloneqq
\mathbb E_{\rho_0,\pi_\theta,\tilde f_\omega}
\left[\sum_{k=0}^{\infty}\gamma^k
\tilde r_\omega(s^k,a_{\mathrm{nom}}^k)\right].
\end{equation}
We optimize \cref{eq:filtered-objective} with PPO \cite{schulman2017proximal}, holding $\omega$ fixed during each policy update.
Safety affects the trajectory through the filter, but no safety term is added to the task reward.

\begin{lemma}[Equivalence under a perfect filter]
\label{prop:ideal-equivalence}
Let $V_{r,F}^{\pi}$ be the return under \cref{eq:filtered-dyn}. Suppose $F$ is perfect on $\mathcal F$, meaning $f(s,F(s,a))\in\mathcal F$ for all $(s,a)\in\mathcal F\times\mathcal A$, and $F(s,a)=a$ whenever $f(s,a)\in\mathcal F$. Under the conditions of \cite[Thm.~1]{oh2026provably},
\begin{equation}
\sup_{\pi}V_{r,F}^{\pi}(s)
=\sup_{\pi:\,V_h^{\pi}(s)\leq0}V_r^{\pi}(s),
\qquad \forall s\in\mathcal F.
\end{equation}
Hence, globally optimizing \cref{eq:filtered-objective} over a policy class containing an optimizer recovers \cref{eq:regime-feasible}.
\end{lemma}
\textit{Proof.} This is a direct application of \cite[Thm.~1]{oh2026provably} to \cref{eq:filtered-dyn}. \hfill$\square$

\begin{remark}
Exact projection onto $\{a:Q_h^\star(s,a)\leq0\}$ yields a perfect filter. Our learned filter only approximates it, so \cref{prop:ideal-equivalence} is not a guarantee for the deployed controller.
\end{remark}

The task policy is updated from on-policy rollouts, while the safety critic, backup actor, and filter are trained off-policy from the replay buffer $\mathcal D$.
Retaining earlier unsafe and near-boundary transitions preserves safety-relevant coverage as the task policy changes and mitigates catastrophic forgetting \cite{zhang2025dgppo}.
We also execute the backup policy $\mu_\psi$ with some probability on some trajectories to avoid out-of-distribution (OOD) $Q_{h,\phi}$ values. These trajectories are only stored in the replay buffer and never used for the PPO update.

\begin{remark}
\label{rem:training-interpretation}

\algo{} is structurally different from objective-based safe RL methods. Those methods optimize return and safety with the \textit{same} policy, so policy updates must reconcile their competing directions. \algo{} optimizes safety and task return in different modules: \cref{eq:filter-loss} updates $F_\omega$ for safety, while PPO updates $\pi_\theta$ for return \cref{eq:filtered-objective}. Thus, no safety term competes with return in the task-policy objective, and training through the filter lets the policy adapt to interventions.

\end{remark}

\section{Related Work}
\label{sec:related}

\textbf{Feasibility-aware safe RL.} Hamilton-Jacobi reachability characterizes states from which future violations can be avoided \cite{fisac2019bridging,wabersich2023data}. Neural reachability solvers approximate the underlying dynamic program, while RL formulations estimate its value from sampled transitions \cite{bansal2021deepreach,feng2025bridging,fisac2019bridging,hsu2021safety}. Subsequent work extends these formulations to stabilize-avoid objectives and unknown feasibility \cite{so2023solving,so2026solving}. Feasibility-aware safe RL then incorporates the learned value into task-policy optimization through multipliers, reachability-dependent objectives, or feasibility-guided generation \cite{yu2022reachability,ganai2023respo,zheng2024safe}. We adopt the same two-regime problem but enforce safety on actions through a correction network.
The task policy optimizes only return under filtered dynamics, without a competing safety objective in its update.

\textbf{Safety filters.} Safety filters decouple task control from safety enforcement by correcting nominal actions before actuation \cite{ames2016control,wabersich2023data}.
To construct a valid filter, learning-based methods either learn the filter function $B$ and retain online filtering \cite{robey2020learning,so2024train,oh2026synthesis,thananjeyan2021recovery,hsu2023isaacs,oh2025safety,kim2026deep}, or learn the correction directly or train the task policy under filtered dynamics \cite{yu2022towards,bejarano2024safety,yang2025cbf,yun2026provably}. \algo{} combines these directions by approximating the optimal state-action peak-harm value, amortizing minimal-intervention correction, and training the task policy on filtered rollouts. Its feedforward filter requires neither a dynamics model nor an online solve at deployment, and the same loss prioritizes lower predicted harm when its learned acceptance set is empty.

\textbf{Safe control for legged and humanoid robots.} Safety filters for legged robots range from reduced-order barrier controllers to full-order methods \cite{cohen2024safety,nguyen2024gameplay}, with recent humanoid work addressing joint, actuation, and link-level constraints \cite{yang2025shield,chen2025dexterous,muraccioli2026safe,yang2026pac}. Complementing direct action correction, learned reachability supports safe-set estimation and recovery-policy switching \cite{he2024agile,chen2026lambdareachability}, while protective-falling policies mitigate robot damage \cite{meng2026safefall,xu2025unified}. \algo{} addresses both avoidance and harm reduction through one model-free filter, with safety specified by any harm function. When safety is feasible, the task policy learns under the filtered dynamics to target the long-horizon constrained RL problem. Otherwise, the same filter prioritizes lower harm. This unified architecture allows protected-region clearance to take precedence over balance without a fall-away specific reward, a mode switch, or a separate recovery controller.

\section{Simulation Experiments}
\label{sec:exp}

We conduct experiments to answer three research questions. \textbf{(Q1):} Starting from feasible states, how much task return can \ourname{} retain while respecting the safety constraint? \textbf{(Q2):} When harm is unavoidable, can \ourname{} minimize the harm? \textbf{(Q3):} Does \ourname{} scale to a high-dimensional system with whole-body constraints? We compare against four representative ways of enforcing safety in the literature:

\begin{itemize}[\setlength{\labelindent}{0.5\parindent}]
  \item \baselinename{RCRL} \cite{yu2022reachability}, a primal-dual feasibility-aware method that enforces a policy-conditioned safety value in the objective.
  \item \baselinename{RC-CRPO}, our feasibility-aware adaptation of the primal method CRPO \cite{xu2021crpo}, using the safety value.
  \item \baselinename{PPO($\beta$)} \cite{schulman2017proximal}, which subtracts $\beta\relu{\harm}$ from the task reward as in \cref{sec:prelim:lagrangian}.
  \item \baselinename{ROCBF($\alpha$)} \cite{cohen2024safety,yang2025cbf}, applying a CBF to a reduced-order single-integrator model tracked by the full system.
\end{itemize}

\subsection{One-dimensional Double Integrator}
\label{sec:exp:di}
\label{sec:exp:infeas}

We first consider a double integrator (DI) with state $\stt=(p,v)$, where $p\in\R$ is position and $v\in\R$ is velocity, with acceleration control $\act\in[-1,1]$. The harm function $\harm(\stt)=\abs{p}-1$ defines the safe states $\abs{p}\leq1$, while the reward $\rew(\stt)=-\min\{\abs{p-2},\abs{p+2}\}$ is maximized outside that set. Safety and performance therefore conflict.
Maximum braking gives the continuous-time optimal $V_h^\star$, and \Cref{fig:toy1d,fig:feas} shows the set $V_h^\star\leq0$ in \gtset{grey}.
We sample initial states with $p,v\sim\mathcal U[-1,1]$ to include both feasible and infeasible states.
For each method, we train $3$ seeds for $2000$ iterations, and select the safest checkpoint among iterations $1800$, $1900$, and $2000$ for evaluation with $4096$ episodes, breaking ties by higher return.
The validation set is independent of the evaluation set.
The results are shown in \Cref{fig:toy1d-safety-return-matched}, with representative trajectories plotted in \Cref{fig:toy1d}.

\begin{figure}[!t]
  \centering
  \includegraphics[width=\columnwidth]{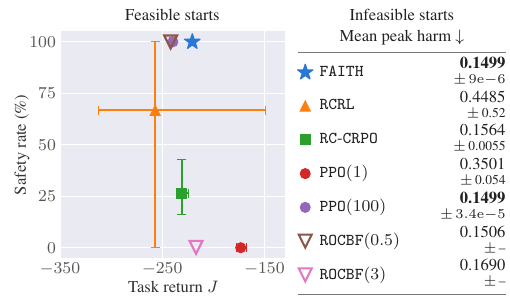}
  \setlength{\abovecaptionskip}{-15pt}
  \caption{
    \textbf{Double integrator results: safety rate and task return of feasible starts, and mean peak harm of infeasible starts.} Left: For feasible starts, \ourname{} achieves a $100\%$ safety rate, and the highest return among all safe methods (error bars show min-max ranges). Right: For infeasible starts, \ourname{} achieves the lowest peak harm ($\pm$ shows the standard deviation).
  }
  \label{fig:toy1d-safety-return-matched}
\end{figure}

\begin{figure*}[!t]
  \centering
  \setlength{\abovecaptionskip}{-5pt}
  \includegraphics[width=\textwidth]{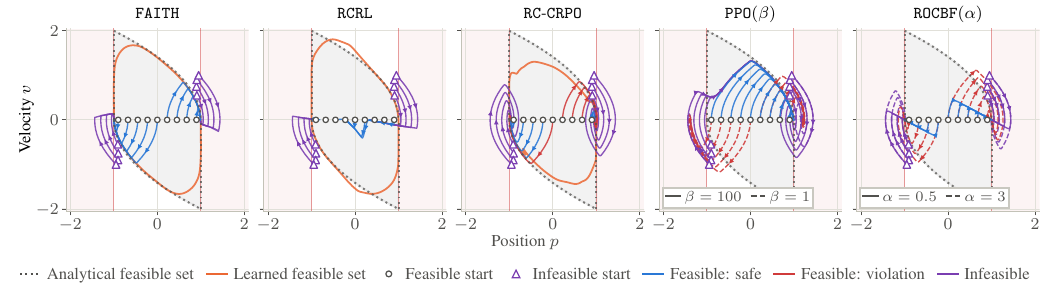}
  \caption{
    \textbf{Trajectories in the double integrator environment.} \ourname{} matches the \gtset{feasible set} closely in visited states. From feasible starts (\textcolor[HTML]{52514E}{$\circ$}), \ourname{} rides the feasibility boundary to achieve the highest return and stays safe. When infeasible (\textcolor[HTML]{7A3EB1}{$\triangle$}), \ourname{} applies maximum braking to minimize peak harm.
  }
  \label{fig:toy1d-trajectories-matched}
  \label{fig:toy1d}
  \label{fig:feas}
\end{figure*}

\textbf{(Q1): Starting from feasible states, \ourname{} maintains $100\%$ safety and achieves the highest return among all safe methods.} \Cref{fig:toy1d} illustrates that \ourname{} achieves this by learning a feasibility boundary close to the analytical feasible set along the \textit{visited} states, and braking near that boundary. \baselinename{RCRL} exhibits unstable performance with a large variance between different seeds, potentially caused by the unstable saddle-point problem. \baselinename{RC-CRPO} crosses the boundary potentially because of the competing policy updates between safety and performance. \baselinename{PPO} exhibits the expected safety-performance trade-off. \baselinename{ROCBF} generates suboptimal trajectories because it ignores the acceleration limit.

\textbf{(Q2): Starting from infeasible states, \ourname{} attains the lowest peak harm.} $350$ states are infeasible in the sampled resets. \Cref{fig:toy1d-safety-return-matched} (right) shows that \ourname{} and \baselinename{PPO($100$)} both attain the lowest mean peak harm. \Cref{fig:toy1d-trajectories-matched} shows that \ourname{} achieves this by braking immediately from these starts.

\subsection{Safety Gym}
\label{sec:exp:car-circle}
To test \ourname{} on a commonly used benchmark, we perform experiments on the SafetyCarCircle1-v0 task in Safety Gym. We train each method for $3$ seeds, select the checkpoints following the same process as \Cref{sec:exp:di}, and report the results in \Cref{fig:car-circle-baselines}. With all initial states being feasible, \ourname{} achieves a $100\%$ safety rate and high mean return, whereas all evaluated baselines incur violations.

\begin{figure}[!t]
  \centering
  \includegraphics[width=\columnwidth]{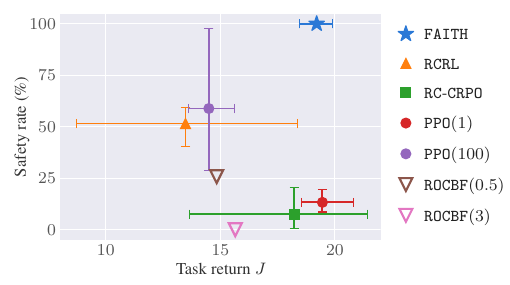}
  \setlength{\abovecaptionskip}{-15pt}
  \caption{\textbf{Car-Circle Results.} With all initial states being feasible, \ourname{} has $100\%$ safety and high task return (error bars show min-max ranges).}
  \label{fig:car-circle-baselines}
\end{figure}

\begin{figure*}[!t]
  \centering
  \includegraphics[width=\textwidth]{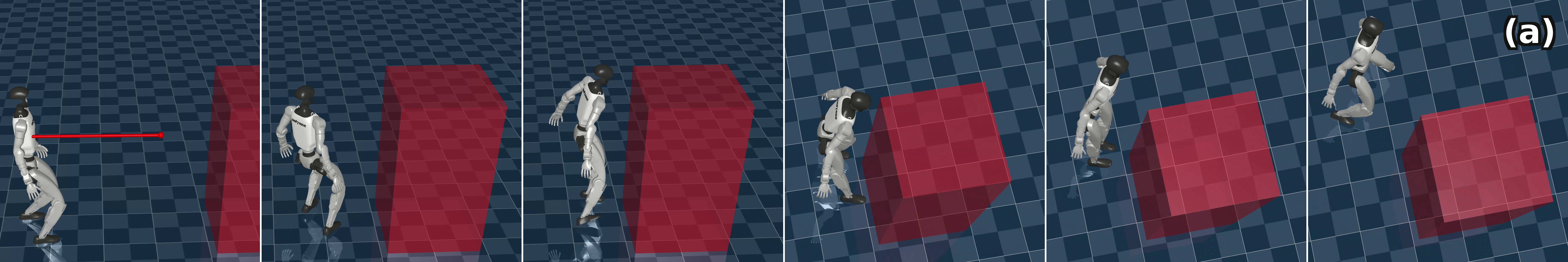}\\[1pt]
  \includegraphics[width=\textwidth]{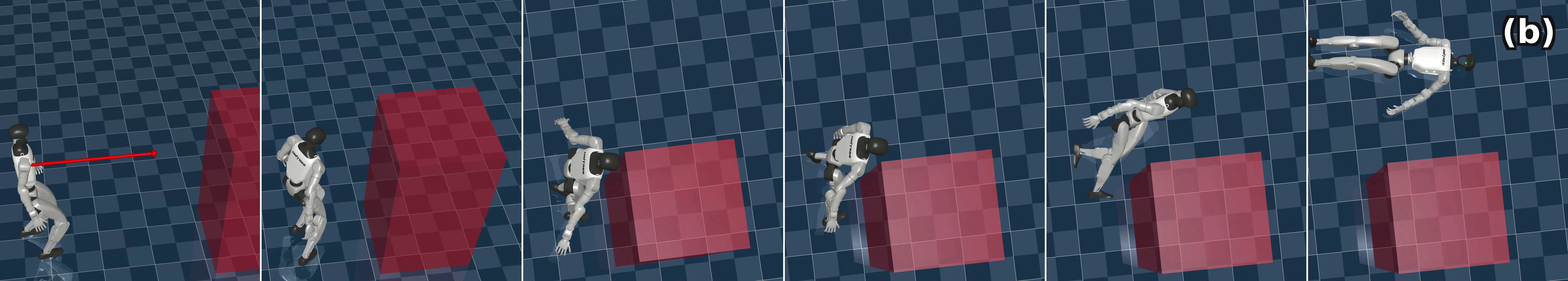}
  \setlength{\abovecaptionskip}{-10pt}
  \caption{\textbf{Push-Avoid simulation.} The red box marks the protected region and the arrow denotes the push. {(a)} The robot remains upright while staying outside the region. {(b)} When the push is too hard, the robot sacrifices balance and falls away while protecting its head and hands.}
  \label{fig:simrollouts}
\end{figure*}

\subsection{$29$-DoF Humanoid Near Protected Regions}
\label{sec:exp:g1}

To test whether \ourname{} scales to high-dimensional systems, we consider a $29$-DoF Unitree G1 humanoid in Isaac Lab \cite{mittal2025isaaclab}, actuated by joint-position targets. We mark planar regions with infinite height as protected, representing, for example, the space occupied by humans\footnote{We consider safety as human protection, so falling is not unsafe.}. We represent the robot's $31$ links by mesh-inflated points and define $\harm(\stt)$ as the largest signed penetration of any point into the protected region. The robot observes the protected region boundaries by an analytical $64$-beam $360^\circ$ planar LiDAR scan.

We design two tasks. \textit{Walking-Avoid}: Track a given velocity command by walking while avoiding protected regions. The robot begins at rest outside the protected regions, so the initial states are feasible by remaining stationary. \textit{Push-Avoid}: On top of Walking-Avoid, the robot is pushed at the start of each episode. Because the ground truth feasible set is unknown here, we evaluate overall safety and separately analyze a \textit{common failure subset}, i.e., initial states where all evaluated methods violate safety, without labeling it physically infeasible.
Both tasks use a reward comprising velocity tracking, foot-slip penalties, and related locomotion terms. We additionally penalize contact forces on the robot, with larger weights on the head and hands and smaller on the torso, to encourage it to protect its most fragile parts when falling \cite{meng2026safefall}.
Over $4096$ trajectories, we report the safety rate, i.e., the fraction of safe trajectories, together with the return, and also the peak penetration on the common failure subset.

We consider two baselines. \baselinename{No Filter} is a pure PPO locomotion policy trained for velocity tracking and push recovery, without safety terms. \baselinename{ROCBF($\cbfAlpha$)} uses \baselinename{No Filter} as a nominal policy but filters its planar velocity command with the reduced-order CBF, which models the base dynamics as a single integrator. We set the body-inflation radius to $r=0.40$\,m and sweep the CBF gain $\cbfAlpha\in\{0,0.1,1,5\}$\footnote{We omit other baselines for their poor performance on a humanoid.}.

\begin{table}[t]
\centering
\caption{
  \textbf{Walking-Avoid.} $\pm$ denotes episode-level standard deviation.
  }
\label{tab:walkavoid}
\footnotesize
\setlength{\tabcolsep}{5pt}
\begin{tabular}{@{}l|cc@{}}
\toprule
Method & Safety \% $\uparrow$ & Return $\uparrow$ \\
\midrule
\baselinename{No Filter}     & 80.37 & $2428 \pm 288$ \\
\baselinename{ROCBF($0$)}    & \textbf{99.95} & $2094 \pm 342$ \\
\baselinename{ROCBF($0.1$)}  & \textbf{99.95} & $2182 \pm 330$ \\
\baselinename{ROCBF($1$)}    & 99.71 & $2362 \pm 307$ \\
\baselinename{ROCBF($5$)}    & 94.80 & $2373 \pm 303$ \\
\ourname{}                   & \textbf{99.95} & $2365 \pm 345$ \\
\bottomrule
\end{tabular}
\end{table}

\begin{table}[t]
\centering
\caption{
  \textbf{Push-Avoid.} Peak harm is evaluated on common failure episodes. $\pm$ denotes episode-level standard deviation.
  }
\label{tab:pushavoid}
\footnotesize
\setlength{\tabcolsep}{2pt}     \begin{tabular}{@{}l|ccc@{}}
\toprule
Method & Safety \% $\uparrow$ & Return $\uparrow$ & $\E\max_k \harm$ (common failure) $\downarrow$ \\
\midrule
\baselinename{No Filter}     & 21.53          & $123.2 \pm 96.0$  & $0.253 \pm 0.149$ \\
\baselinename{ROCBF($0$)}    & 84.52          & $115.5 \pm 64.3$  & $0.172 \pm 0.146$ \\
\baselinename{ROCBF($0.1$)}  & 83.35          & $116.3 \pm 41.3$  & $0.175 \pm 0.140$ \\
\baselinename{ROCBF($1$)}    & 78.56          & $117.1 \pm 87.0$  & $0.165 \pm 0.140$ \\
\baselinename{ROCBF($5$)}    & 60.11          & $115.5 \pm 96.6$  & $0.166 \pm 0.133$ \\
\ourname{}                   & \textbf{85.45} & $\phantom{00}80.0 \pm 104.9$ & $\mathbf{0.153} \pm 0.184$ \\
\bottomrule
\end{tabular}
\end{table}

\textbf{(Q1, Q3): \ourname{} matches the highest safety rate in Walking-Avoid while retaining $97\%$ of unfiltered return, and achieves the highest safety rate in Push-Avoid.}
At matched safety in Walking-Avoid, \ourname{} yields higher return than \baselinename{ROCBF($0$)} and \baselinename{ROCBF($0.1$)} (\cref{tab:walkavoid}). In Push-Avoid, \ourname{}'s higher safety comes with lower return as it sacrifices balance to preserve clearance because it does not trade safety for higher return (\cref{tab:pushavoid,fig:simrollouts}).

\textbf{(Q2, Q3): In the common failure Push-Avoid episodes, \ourname{} obtains the lowest harm.} The same task-policy and filter pair of \ourname{} limits full-body penetration without a fallback controller or explicit mode switch (\cref{tab:pushavoid}).

\subsection{Ablation Studies}

\label{sec:exp:ablation}
\begin{table}[t]
\centering
\caption{\textbf{Double Integrator ablations.}
$\pm$ denotes standard deviation.}
\label{tab:ablation-summary}
\begingroup
\newcommand{\ablcontroller}[1]{\mbox{#1}}
\newcommand{\ablpm}[2]{$#1{\scriptstyle\pm #2}$}
\setlength{\tabcolsep}{0.9pt}
\renewcommand{\arraystretch}{1.0}
\begin{tabular*}{\columnwidth}{@{\extracolsep{\fill}}l|rrr@{}}
\toprule
Controller & \shortstack[r]{Return $J\uparrow$} & \shortstack[r]{Safety (\%) $\uparrow$} & \shortstack[r]{$\mathbb E\max_k h$ $\downarrow$} \\
\midrule
\ablcontroller{\baselinename{Oracle}} & $-219.37$ & $100.00$ & $0.1499$ \\
\ablcontroller{\ourname{}} & \ablpm{-220.64}{0.23} & $100.00$ & \ablpm{0.1499}{9\mathrm{e}{-6}} \\
\ablcontroller{\baselinename{PPO+GTCBF}} & \ablpm{-219.40}{0.01} & $100.00$ & $0.1500$ \\
\ablcontroller{\baselinename{PPO+ROCBF($0.5$)}} & \ablpm{-241.39}{0.01} & $99.95$ & $0.1506$ \\
\ablcontroller{\baselinename{PPO+ROCBF($3$)}} & \ablpm{-187.46}{0.04} & $0.00$ & \ablpm{1.3501}{1.7\mathrm{e}{-4}} \\
\ablcontroller{\baselinename{PPO}\textcolor{MaterialPurple900}{+}\baselinename{FAITH} \textcolor{MaterialPurple900}{filter (fixed)}} & \ablpm{-221.09}{0.22} & \ablpm{99.98}{0.04} & \ablpm{0.1499}{9\mathrm{e}^{-6}} \\
\ablcontroller{\baselinename{FAITH} \textcolor{MaterialPurple900}{w/o filtered dyn.}} & \ablpm{-230.45}{1.71} & \ablpm{97.52}{2.74} & \ablpm{0.1660}{0.02} \\

\ablcontroller{\baselinename{FAITH} \textcolor{MaterialPurple900}{w/o backup replay}} & \ablpm{-221.63}{2.98} & \ablpm{83.87}{27.92} & \ablpm{0.1505}{0.001} \\
\ablcontroller{\baselinename{FAITH} \textcolor{MaterialPurple900}{w/ numerical search}} & \ablpm{-220.65}{0.10} & $100.00$ & \ablpm{0.1500}{1\mathrm{e}{-4}} \\
\ablcontroller{\baselinename{FAITH} \textcolor{MaterialPurple900}{($0.25\lambda_{\max}$)}} & \ablpm{-208.23}{10.46} & \ablpm{0.01}{0.01} & \ablpm{0.3313}{0.11} \\
\ablcontroller{\baselinename{FAITH} \textcolor{MaterialPurple900}{($0.5 \lambda_{\max}$)}} & \ablpm{-220.45}{0.22} & \ablpm{99.99}{0.01} & \ablpm{0.1609}{0.019} \\
\ablcontroller{\baselinename{FAITH} \textcolor{MaterialPurple900}{($2\lambda_{\max}$)}} & \ablpm{-221.42}{0.44} & $100.00$ & \ablpm{0.1499}{7.7\mathrm{e}{-6}} \\
\bottomrule
\end{tabular*}
\endgroup
\end{table}
\begin{table}[t]
\centering
\caption{\textbf{Car-Circle ablations.} $\pm$ denotes standard deviation.}
\label{tab:car-circle-ablation-summary}
\begingroup
\newcommand{\ablcontroller}[1]{\mbox{#1}}
\newcommand{\ablpm}[2]{$#1{\scriptstyle\pm #2}$}
\setlength{\tabcolsep}{0.9pt}
\renewcommand{\arraystretch}{1.0}
\begin{tabular*}{\columnwidth}{@{\extracolsep{\fill}}l|rr@{}}
\toprule
Controller & \shortstack[r]{Return $J\uparrow$} & \shortstack[r]{Safety (\%) $\uparrow$} \\
\midrule
\ablcontroller{\ourname{}} & \ablpm{19.20}{0.71} & $100.00$ \\
\ablcontroller{\baselinename{PPO}\textcolor{MaterialPurple900}{+}\baselinename{FAITH} \textcolor{MaterialPurple900}{filter (fixed)}} & \ablpm{18.46}{1.39} & $100.00$ \\
\ablcontroller{\baselinename{FAITH} \textcolor{MaterialPurple900}{w/o filtered dyn.}} & \ablpm{8.96}{5.21} & \ablpm{98.44}{2.07} \\
\ablcontroller{\baselinename{FAITH} \textcolor{MaterialPurple900}{w/o backup replay}} & \ablpm{18.17}{1.29} & \ablpm{97.66}{3.41} \\
\ablcontroller{\baselinename{FAITH} \textcolor{MaterialPurple900}{w/ numerical search}} & \ablpm{18.84}{0.37} & \ablpm{96.88}{2.07} \\
\ablcontroller{\baselinename{FAITH} \textcolor{MaterialPurple900}{($0.25\lambda_{\max}$)}} & \ablpm{18.98}{0.93} & \ablpm{95.31}{5.63} \\
\ablcontroller{\baselinename{FAITH} \textcolor{MaterialPurple900}{($0.5\lambda_{\max}$)}} & \ablpm{18.50}{1.55} & \ablpm{99.48}{0.90} \\
\ablcontroller{\baselinename{FAITH} \textcolor{MaterialPurple900}{($2\lambda_{\max}$)}} & \ablpm{17.92}{0.88} & $100.00$ \\
\bottomrule
\end{tabular*}
\endgroup
\end{table}

We conduct ablations to study the design choices in the DI (\Cref{tab:ablation-summary}) and the Car-Circle environments (\Cref{tab:car-circle-ablation-summary}).

\begin{figure*}[t]
  \centering
  \includegraphics[width=\textwidth]{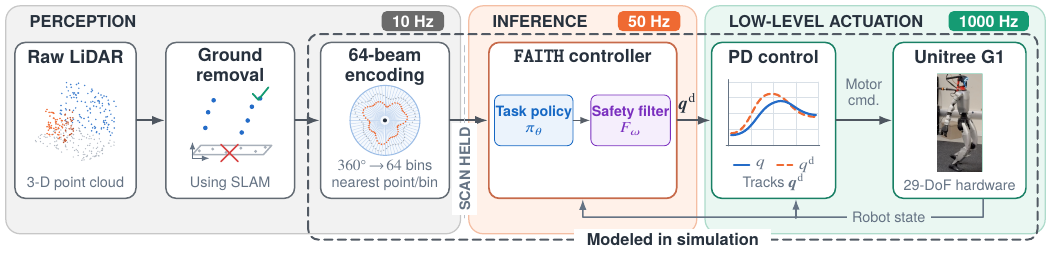}
  \setlength{\abovecaptionskip}{-13pt}
  \caption{\textbf{Hardware control pipeline.} SLAM-based ground removal and angular binning produce a $64$-beam LiDAR observation. \ourname{} generates joint targets tracked by PD control. The dashed box marks simulated components, including G1 dynamics and analytically generated beam inputs.}
  \label{fig:hardware-control}
\end{figure*}

\textbf{The learned filter approaches the analytical reference.}
We introduce $3$ more baselines in DI: \baselinename{Oracle} applies the analytical optimal controller; \baselinename{PPO+GTCBF} trains PPO through a ground-truth CBF; \baselinename{PPO+ROCBF($\alpha$)} trains PPO through a ROCBF. We observe that \ourname{} closely follows the analytical optimal solution produced by \baselinename{Oracle} and \baselinename{PPO+GTCBF}, outperforming training \baselinename{PPO} through \baselinename{ROCBF}.

\textbf{The learned filter is transferable.}
Fixing the filter of \ourname{} and training PPO through it (\baselinename{PPO}\textcolor{MaterialPurple900}{+}\baselinename{FAITH} \textcolor{MaterialPurple900}{filter (fixed)}) performs similarly to \ourname{} in both environments.

\textbf{Filtered-dynamics training improves the performance.}
We remove training through the filter dynamics described in \Cref{sec:method:policy} (\baselinename{FAITH} \textcolor{MaterialPurple900}{w/o filtered dyn.}), and observe that task return degrades because PPO does not adapt to the filter. It also slightly degrades safety potentially because the policy tends to output more unsafe actions that fill the replay buffer, making it more difficult to learn a good filter.

\textbf{Backup-generated replay improves safety.}
Replacing backup rollouts with filtered task-policy rollouts (\baselinename{FAITH} \textcolor{MaterialPurple900}{w/o backup replay}) lowers safety, potentially because $Q_{h,\phi}$ can be OOD without the backup replay (\Cref{sec:method:policy}).

\textbf{The learned filter outperforms numerical search.}
With the policy and critic fixed, \baselinename{FAITH} \textcolor{MaterialPurple900}{w/ numerical search} retains $a_{\mathrm{nom}}$ if $Q_{h,\phi}(s,a_{\mathrm{nom}})\leq-\delta$, and otherwise samples $1025$ actions and chooses the closest to $a_{\mathrm{nom}}$ with $Q_{h,\phi}<-\delta$. If none qualifies, it selects the sample minimizing $Q_{h,\phi}$. It performs similarly to \ourname{} in DI but worse in Car-Circle, and is $18\times$ slower in both environments.

\textbf{Safety requires a sufficiently large $\lambda_\mathrm{max}$.}
Reducing the maximum penalty (\baselinename{FAITH} \textcolor{MaterialPurple900}{($0.25\lambda_\mathrm{max}$)}) degrades safety and increases peak harm. Increasing it (\baselinename{FAITH} \textcolor{MaterialPurple900}{($2\lambda_\mathrm{max}$)}) matches the safety rate, but yields slightly lower return.
The results remain robust to moderate changes (\baselinename{FAITH} \textcolor{MaterialPurple900}{($0.5\lambda_\mathrm{max}$)}).

\section{Hardware Experiments}
\label{sec:hardware}

We conduct qualitative hardware demonstrations on a $29$-DoF Unitree G1, with boxes representing the protected region (\cref{fig:teaser}). In the pipeline (\cref{fig:hardware-control}), $10$\,Hz LiDAR scans undergo SLAM-based ground removal \cite{chen2023direct}. We retain the nearest point in each of $64$ equal-width azimuth bins covering $360^\circ$, matching the simulation observation format. At $50$\,Hz, \ourname{} combines this observation (held between scans) with the current robot state to generate joint-position targets. A $1000$\,Hz PD controller then tracks these targets.

\textbf{The same controller supports upright recovery and falling away from the protected region.} Under a lighter push, the robot recovers balance while avoiding the boxes (\cref{fig:teaser}(a)). Under a stronger push, it falls away onto its legs and torso (\cref{fig:teaser}(b)),  consistent with the head and hand contact penalties used in training (\cref{sec:exp:g1}). 

\section{Conclusions}

We introduced \ourname{}, a feasibility-aware, model-free framework that separates safety enforcement from task-policy optimization. \algo{} approximates the optimal safety critic and trains PPO through a learned filter with only task reward. Safety and return therefore do not compete. Under stated assumptions, filtering preserves the constrained optimum. When the learned feasible set is empty, the filter approaches minimum predicted peak harm.
Simulations demonstrate safety, performance, and harm reduction, while experiments on a humanoid demonstrate hardware execution.

\textbf{Limitations.} Discounting, finite $\lambda$, and learning errors leave the learned controller without formal safety or optimality guarantees. Evaluation is limited to planar protected regions and excludes impact severity and human injury. Hardware demonstrations are qualitative, without safety or harm statistics or quantified effects of LiDAR noise, control latency, and communication delay.

\section*{Acknowledgments}

We thank Nicholas Morozovsky, Yinai Fan, and Shravan Tangudu for proposing the humanoid test cases, and Zeyuan Feng and Nick Cerone for help with hardware experiments.

\bibliographystyle{IEEEtran}
\bibliography{ref}

\end{document}